\documentclass{llncs}
\usepackage{llncsdoc}
\usepackage{amsmath}%
\usepackage{amsfonts}%
\usepackage{amssymb}%
\usepackage[pdftex]{graphicx}%
\usepackage{float}%
\usepackage{subfigure}%
\usepackage{multicol}%
\usepackage{epstopdf}

\usepackage[ruled,,vlined,lined,linesnumbered]{algorithm2e}

\usepackage{graphicx}
\usepackage{tikz}
\usetikzlibrary{calc}
\usepackage{subfigure}
\usepackage{caption}

\newcommand{\norm}[1]{\left\lVert#1\right\rVert}
	
\begin{document}
\title{Application of the Ring Theory in the Segmentation\\ of Digital Images}

\author{Yasel Garc\'es,\ Esley Torres,\ Osvaldo Pereira and Roberto Rodr\'iguez}
\institute{Institute of Cybernetics, Mathematics and Physics, Havana, Cuba\\ \quad \email{88yasel@gmail.com},
\email{esley@icimaf.cu},
\email{rrm@icimaf.cu}}
\maketitle

\begin{abstract}
Ring theory is one of the branches of the abstract algebra that has been broadly used in images. However, ring theory has not been very related with image segmentation. In this paper, we propose a new index of similarity among images using $\mathbb{Z}_{n}$ rings and the entropy function. This new index was applied as a new stopping criterion to the Mean Shift Iterative Algorithm with the goal to reach a better segmentation. An analysis on the peformance of the algorithm with this new stopping criterion is carried out. The obtained results proved that the new index is a suitable tool to compare images. 
\end{abstract}

%\keywords{Ring theory, segmentation, digital images, mean shift, entropy.}

\section{Introduction}
Many techniques and algorithms have been proposed for digital image segmentation. Traditional segmentation such as thresholding, histograms or other conventional operations are rigid methods. Automation of these classical approximations is difficult due to the complexity in shape and variability within each individual object in the image.

The mean shift is a non-parametric procedure that has demostrated to be an extremely versatile tool for feature analysis. It can provide reliable solutions for many computer vision tasks \cite{Comaniciu00}. Mean shift method was proposed in 1975 by Fukunaga and Hostetler \cite{Fukunaga}. It was largely forgotten until Cheng's paper retook interest on it \cite{Cheng95}. Segmentation by means of the Mean Shift Method carries out as a first step a smoothing filter before segmentation is performed \cite{Comaniciu00,Comaniciu02}.

Entropy is an essential function in information theory and this has had a special uses for images data, e.g., restoring images, detecting contours, segmenting images and many other applications \cite{Suyash06,Zhang03}. However, in the field of images the range of properties of this function could be increased if the images are defined in $\mathbb{Z}_{n}$ rings. The inclusion of the ring theory to the spatial analysis is achieved considering images as a matrix in which the elements belong to the cyclic ring $\mathbb{Z}_{n}$. From this point of view, the images presents cyclical properties associated to gray level values.

Ring Theory has been well-used in cryptography and many others computer vision tasks \cite{Neiderreiter}. The inclusion of ring theory to the spatial analysis of digital images, it is achieved considering the image like a matrix in which the elements belong to finite cyclic ring $\mathbb{Z}_{n}$. The ring theory for the Mean Shift Iterative Algorithm was employed by defining images in a ring $\mathbb{Z}_{n}$. A good performance of this algorithm was achieved. Therefore, the use of the ring theory could be a good structure when one desires to compare images, due to that the digital images present cyclical properties associated with the pixel values. This property will allow to increase or to diminish the difference among pixels values, and will make possible to find the edges in the analyzed images.

In this paper, a new similarity index among images is defined, and some interesting properties based on this index are proposed. We compare also the instability of the iterative mean shift algorithm ($\mathit{MSHi}$)  by using this new stopping criterion with regard to the stopping criterion used in \cite{Rodriguez11,Rodriguez11a,Rodriguez12,Rodriguez08}. Furthermore, we make an extension of \cite{Yasel13}, and we expand the theoretical aspects by studying in depth the cyclical properties of rings applied to images. For this purpose, and in order to mark the difference of this paper with regard to \cite{Yasel13}, some issues are pointed out below: 

\begin{itemize}
\item Revision of the mean shift theory.
\item  Important elements of the ring $G_{k\times m} (\mathbb{Z}_{n})(+,\cdot)$ are given: neutral, unitary, and inverse. In particular, the inverse element was used so much to the theoretical proofs as well as practical aspects. 
\item Explanation of strong equivalent images by using histograms.
\item Definition of equivalence classes.
\item Quotient space. Definition and existence.
\item Natural Entropy Distance (NED) definition.
\item  Configuration of the algorithm $\mathit{MSHi}$ with the NED distance.
\end{itemize}

The remainder of the paper is organized as follows. In Section \ref{Theorical aspects}, the more significant theoretical aspects of the mean shift and entropy are given. Section \ref{Similarity} describes the similarity index, its consecuences for entropy function and the significance of the cyclic ring $\mathbb{Z}_{n}$ for images. Also, it is defined the quotient space of strongly equivalent images and some properties of entropy are proved. The experimental results, comparisons and discussion are presented in Section \ref{experiments}. Finally, in Section \ref{conclusion} the conclusions are given.

%\textbf{Falta terminar la parte de la organizacion del trabajo, poner aca cuando se hable de la seccion del algoritmo iterativo de la media desplazada que las siglas son MSHi.}

\section{Theoretical Aspects}
\label{Theorical aspects}

\subsection{Mean Shift}
\label{sub:Mean Shift}
One of the most popular nonparametric density estimators is kernel density estimation. Mathematically, the general multivariate kernel density estimate at the point $x$, is defined by:
\begin{eqnarray}
f(x) =\frac{1}{nh^d}\sum\limits_{i=1}^{n}K\left(\frac{x-x_{i}}{h}\right) .
\end{eqnarray}
where $n$ data points $x_{i},\ i=1,\ldots,n$ represent a population with some unknown density function $f(x)$.

In the case of images:
\begin{align}
\label{eq:caso imagenes kernel}
K_{H}(x)=\vert H \vert^{-1/2}K\left( H^{-1/2}x \right) 
\end{align}
where $K(z)$ is the $d-variate$ kernel function with compact support satisfying the regularity constraints as described in \cite{Wand95}, and $H$ is a symmetric positive definite bandwidth matrix.

For image segmentation, the feature space is composed of two independent domains: the $spatial/lattice$ domain and the $range/color$ domain. We map a pixel to a multidimensional feature point which includes the $p$ dimensional spatial lattice ($p = 2$ for image) and $q$ dimensional color ($q = 1$ for gray scale images and $q = 3$ for color image and $q > 3$ for multispectral image). Due to the different natures of the domains, the kernel is usually broken into the product of two different radially symmetric kernels (subscript $s$ is refer to the spatial domain, and $r$ to the color range):
\begin{align}
\label{imagen kernel defintion}
K_{h_{s},h_{r}}(x)=\frac{c}{(h_{s})^{p}(h_{r})^{q}}k_{s}\left( \norm{\frac{x}{h_{s}} }^{2} \right) k_{r}\left( \norm{\frac{x}{h_{r}}}^{2} \right)
\end{align}
where $x$ is a pixel, $k_{s}$ and $k_{r}$ are the profiles used in the two respective domains, $h_{s}$ and $h_{r}$ are employed bandwidths in $spatial-range$ domains and $c$ is the normalization constant. Using the equation (\ref{imagen kernel defintion}), the kernel density estimator is:
\begin{align}
\label{eq:imagen density estimation}
\hat{f}(x)=\frac{c}{n(h_{s})^{p}(h_{r})^{q}}\sum_{i=1}^{n} k_{s}\left( \norm{\frac{x-x_{i}}{h_{s}}}^{2} \right) k_{r}\left( \norm{\frac{x-x_{i}}{h_{r}}}^{2} \right).
\end{align}
As was shown in (\ref{imagen kernel defintion}) and (\ref{eq:imagen density estimation}), there are two main parameters that have to be defined by the user: the spatial bandwidth $h_{s}$ and the range bandwidth $h_{r}$.

The {\itshape Epanechnikov} function is chosen as the kernel function in this work, this function is defined as: 

\begin{eqnarray}
K_E(x)=\left\{\begin{array}{ll}
\frac{1}{2} c^{-1}_{d} (d+2)\left( 1-\left\|x\right\|^2\right),  & \mbox{if}\ \left\|x\right\|<1\\
& \\
0, & \mbox{otherwise}.\\
\end{array}
\right.
\end{eqnarray}

The gradient of function $f(x)$ is formulated as

\begin{eqnarray}
\widehat{\nabla} f(x) = \nabla \widehat{f(x)} =\frac{1}{nh^d}\sum\limits_{i=1}^{n}\widehat K\left( \frac{x-x_{i}}{h}\right) ,
\end{eqnarray}

\begin{align}
\label{label4}
\widehat{\nabla}f_E(x)&=
\frac{1}{n(h^d c_d)}\frac{d+2}{h^2}\sum\limits_{x_i \in S_h(x)} (x_i-x)\nonumber\\
&=\frac{n_x}{n(h^d c_d)}\frac{d+2}{h^2}\frac{1}{n_x}\sum\limits_{x_i \in S_h(x)} (x_i-x),
\end{align}
where region $S_{h}(x)$ is a hypersphere of radius $h$ having volume $h^{d}c_{d}$, centred at $x$, and containing $n_{x}$ data points; that is, the uniform kernel. By simplicity, in this paper, we only approach the case $d=1$ corresponding to gray level images, but could be extended in the same way to color ($d=3$) and multispectral ($d>3$) images. In addition, the last factor in expression (\ref{label4}) is called the sample mean shift

\begin{align}
\label{Mean Shift Epanechnikov}
M_{h,U}(x)=\underbrace{\frac{1}{n_x}\sum\limits_{x_i \in S_h(x)}(x_i-x)}_\text{\large mean of shift values}= \underbrace{\left(\frac{1}{n_x}\sum\limits_{x_i \in S_h(x)}x_i\right)-x}_\text{\large mean shift}
\end{align}

The quantity $\displaystyle\frac{n_x}{n(h^{d} c_{d})}$ is the kernel density estimate $\widehat{f_U}(x)$ (where $U$ means the uniform kernel) computed with the hyper sphere $S_{h}(x)$, and thus we can write the expression (\ref{label4}) as:

\begin{eqnarray}\label{label6}
\widehat{\nabla}f_{E}(x)=\widehat{f_{U}}(x)\frac{d+2}{h^2} M_{h,U}(x)
\end{eqnarray}
which yields,
\begin{eqnarray}\label{label7}
M_{h,U}(x)=\frac{h^2}{d+2}\frac{\widehat{\nabla}f_E(x)}{\widehat{f_{U}}(x)}.
\end{eqnarray}

Expression (\ref{label7}) shows that an estimate of the normalized gradient can be obtained by computing the sample mean shift in a uniform kernel centered on $x$. In addition, the mean shift has the gradient direction of the density estimate at point $x$. Since the mean shift vector always points towards the direction of the maximum density increase, it can define a path leading to a local density maximum; that is, toward the density mode. 

A generalization to others kernels is achieved using {\it profile} definition and {\it shadow kernel} definition. Moreover, a direct relationship settles down between the kernel  used for mean shift vector $M_h(x)$ and the one used for the probabiity density function \cite{Cheng95}. {\it Profile} and {\it shadow} kernels are two suitable definitions to prove important and relevant properties for kernels and mean shift vector. 

In \cite{Comaniciu00}, it was proved that the obtained {\itshape mean shift procedure} by the following steps, guarantees the convergence:
	
\begin{itemize}
\item computing the mean shift vector $M_h(x)$
\item	translating the window $S_h(x)$ by $M_h(x)$.
\end{itemize}

Therefore, if the individual mean shift procedure is guaranteed to converge, a recursively procedure of the mean shift also converges. Other related works with this issue can be seen in \cite{Grenier06,Suyash06}.

\subsection{Entropy}
\label{sub:Entropy}
Entropy is a measure of unpredictability or information content. In the space of the digital images the entropy is defined as:
\begin{definition}[Image Entropy]
The entropy of the image $\mathcal{A}$ is defined by
\begin{equation}
\label{entropy definition}
E(\mathcal{A})=-\sum_{x=0}^{2^{B}-1}p_{x}log_{2}{p_{x}},
\end{equation}
where $B$ is the total quantity of bits of the digitized image $\mathcal{A}$ and $p(x)$ is the probability of occurrence of a gray-level value.
\end{definition} 

By agreement $\log_{2}(0)=0$ \cite{Shannon48}.Within a totally uniform region, entropy reaches the minimum value. Theoretically speaking, the probability of occurrence of the gray-level value, within a uniform region is always one. In practice, when one works with real images the entropy value does not reach, in general, the zero value. This is due to the existent noise in the image. Therefore, if we consider entropy as a measure of the disorder within a system, it could be used as a good stopping criterion for an iterative process, by using $\mathit{MSHi}$. More goodness on entropy applied to image segmentation algorithm can be seen in \cite{Rodriguez11a,Zhang03}.

\section{Similarity Index with Entropy Function}
\label{Similarity}
In recent works \cite{Rodriguez11,Rodriguez11a,Rodriguez12,Rodriguez08}, the entropy has been an important point in order to define a similarity index to compute the difference between two images. One of the most common criterion are shown in the next equation:
\begin{equation}
\label{old criterion}
\nu(\mathcal{A},\mathcal{B})=\vert E(\mathcal{A})-E(\mathcal{B}) \vert , 
\end{equation}
where $E(\cdot)$ is the function of entropy and the algorithm is stopped when $\nu(\mathcal{A}_{k},\mathcal{A}_{k-1})\leq \epsilon$. Here $\epsilon$ and $k$ are respectively the threshold to stop the iterations and the number of iterations.

\begin{definition}[Weak Equivalence in Images]
%\label{weakly equivalent}
Two images $\mathcal{A}$ and $\mathcal{B}$ are weakly equivalents if
$E(\mathcal{A})=E(\mathcal{B})$. We denote the weak equivalence between $\mathcal{A}$ and $\mathcal{B}$ using $\mathcal{A}\asymp\mathcal{B}$.
\end{definition}
\vspace{-.2in}
\parbox{9cm}{
Figure \ref{different images comparison} shows two different images of $64\times 64$ bits. A reasonable similarity index should present a big difference between Figure \ref{one_a} and Figure \ref{two_b}. However, by using the expression (\ref{old criterion}), we obtain that: 
$$\nu(Figure\ \ref{one_a},\ Figure\ \ref{two_b})=0.
$$}\hspace{-1.7in}\parbox{3cm}{
\begin{figure}[H]
\centering
\subfigure[]{\includegraphics[scale=1]{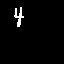}\label{one_a}}
\subfigure[]{\includegraphics[scale=1]{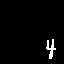}\label{two_b}}
\vspace{-.2in}
\caption{Dissimilar images}
\label{different images comparison}
\end{figure}}

The similarity index in (\ref{old criterion}) never consider the spatial information between the images $\mathcal{A}$ and $\mathcal{B}$. For this reason, it is possible to have two very different images and to obtain a small value by using (\ref{old criterion}). This is a strong reason to consider that the similarity index (\ref{old criterion}) is not appropriated to estimate the distance between two images. 

Taking into account the issues raised above, it is necessary consider a new similarity index based on the following conditions:
\begin{enumerate}
\item{Consider the use of the entropy function.}
\item{Take into account the spatial information between both images.}
\end{enumerate}

The two following subsections have been dedicated to how face this problem by means of the employment of the ring theory. 

\subsection{Image via ring theory}
\label{Image via Ring Theory}

It is natural to think that two images are similar if their subtraction is close to zero. The problem of this idea is that, in general, when the subtraction gives negative values many authors consider to truncate to zero these elements. This consideration, in general, it not describe the difference between two images, and in some cases, it is possible to lose important information. For this reason, it is necessary to define a structure such that the operations between two images are intern.
\begin{definition}[$\mathbb{Z}_{n}$ Ring]
The $\mathbb{Z}_{n}$ ring is the partition of $\mathbb{Z}$ (set of integer numbers) in which the elements are related by the congruence module $n$.
\end{definition}

Mathematically speaking, we say that $a$ is in the class of $b$ ($a\in C_{b}$) if $a$ is related by $(\sim )$ with $b$, where
\begin{eqnarray*}
a\sim b &\Longleftrightarrow a \equiv b (mod\ n) \overset{def}{\Longleftrightarrow} (b-a)\in n\mathbb{Z},\quad \mbox{where}\\
n\mathbb{Z} &= \left\lbrace 0,n,2n,\ldots \right\rbrace \quad \mbox{and}\quad n\in \mathbb{Z}\quad \mbox{is fixed}. 
\end{eqnarray*}
Consequently $\mathbb{Z}_{n} = \lbrace C_{0},\ C_{1},\ldots ,\ C_{n-1}\rbrace$.

We use the structure of the $\mathbb{Z}_{n}$ ring in the set of the images of size $k\times m$ where the pixel values are intergers belonging to $[0,n-1]$ and we denote this set as $G_{k\times m}(\mathbb{Z}_{n})$. We obtain the next result.
\begin{theorem}
%\label{theorem ring matrix}
The set $G_{k\times m}(\mathbb{Z}_{n})(+,\cdot)$, where $(+)$ and $(\cdot )$ are respectively the pixel-by-pixel sum and multiplication in $\mathbb{Z}_{n}$, has a ring structure.
\end{theorem}
\begin{proof}
As the pixels of the image are in $\mathbb{Z}_{n}$, they satisfy the ring axioms. The operation between two images was defined pixel by pixel, then it is trivial that $G_{k\times m}(\mathbb{Z}_{n})$ under the operations $(+,\cdot)$ of the $\mathbb{Z}_{n}$ ring inherits the ring structure.
\end{proof} 

The following observations are important in order to understand the main features of the ring $G_{k\times m} (\mathbb{Z}_{n})(+,\cdot)$:

\begin{itemize}
\item{The neutral element for the sum is the image with all pixels have zero value, it is denoted by $\mathcal{O}$.}
\item{Note that the image with all pixel values are equal to one ($\mathcal{I}$) is the neutral element respect to the multiplication, (observe that the multiplication between two images was defined element by element).}
\item{If $\mathcal{A}\in G_{k\times m} (\mathbb{Z}_{n})(+,\cdot)$ then $E(\mathcal{A})=E(-\mathcal{A})$, where $-\mathcal{A}$ is the additive inverse of $\mathcal{A}$ in $G_{k\times m} (\mathbb{Z}_{n})(+,\cdot)$.}
\end{itemize}

Working with images with coefficients in $\mathbb{Z}_{n}$, one has a better way to analyze images, and the cyclical effect opens a new way that largely reflects the resemblance of a gray level of a pixel with regarding its neighbors.
\begin{definition}[Strong Equivalence in Images]
Two images $\mathcal{A}, \mathcal{B} \in G_{k\times m}(\mathbb{Z}_{n})(+,\cdot)$ are strongly equivalents if
$$
\mathcal{A}=\mathcal{S}+\mathcal{B},
$$
where $\mathcal{S}$ is a scalar image. We denote the strong equivalence between $\mathcal{A}$ and $\mathcal{B}$ as $\mathcal{A}\cong \mathcal{B}$.
\end{definition}

Note that if $\mathcal{A}=\mathcal{S}+\mathcal{B} \Rightarrow \exists\   \overline{\mathcal{S}}\ \vert\ \mathcal{B}=\overline{\mathcal{S}}+\mathcal{A}$ and $\overline{\mathcal{S}}=-(\mathcal{S})$, where $-(\mathcal{S})$ is the additive inverse of $\mathcal{S}$ in the ring. This is calculated using the inverse of each pixels of $\mathcal{S}$ in $\mathbb{Z}_{n}$.

\begin{theorem}
If two images $\mathcal{A}$ and $\mathcal{B}$ are strongly equivalents then they are weakly equivalents.
\end{theorem}
\begin{proof}
If $\mathcal{A}$ and $\mathcal{B}$ are strongly equivalents then $\mathcal{A}=\mathcal{S}+\mathcal{B}$ where $\mathcal{S}$ is a scalar image. Then $E(\mathcal{A}) =E(\mathcal{S}+\mathcal{B})$ but $\mathcal{S}$ is a scalar image and for this reason the sum $\mathcal{S}+\mathcal{B}$ only change in $\mathcal{B}$ the intensity of each pixel, but this does not change the number of different intensities or the frequency of each intensity in the image. Then, $E(\mathcal{S}+\mathcal{B}) =E(\mathcal{B})$. Finally we obtain that $E(\mathcal{A})=E(\mathcal{B})$ and they are weakly equivalents.
\end{proof}

Note that the shown images in Figure \ref{different images comparison} are weakly equivalents, but they are not strongly equivalents. This is an example that $\mathcal{A}\asymp \mathcal{B} \nRightarrow \mathcal{A}\cong \mathcal{B}$. In general, it is important to understand that two images strongly equivalents have the same histogram of frecuency, except for one uniform traslation of all gray levels.

Consider the next example to see the importance of the ring theory in the operations among images. The image in Figure \ref{histogram original image} has a histogram of frecuency that is shown in Figure \ref{Histograma imagenen de ejemplo}. If it is compute the addition and subtraction of the Figure \ref{histogram original image} by a scalar image $\mathcal{S}$ (where $\mathcal{S}$ has all pixels equal to $100$), using the ring theory are obtained the histograms of frecuency that are shown in Figure \ref{frecuencia ring sum} and Figure \ref{frecuencia ring rest}.

\begin{figure}[H]
\centering
\subfigure[Original Image]{\includegraphics[width=4cm, height=3cm] {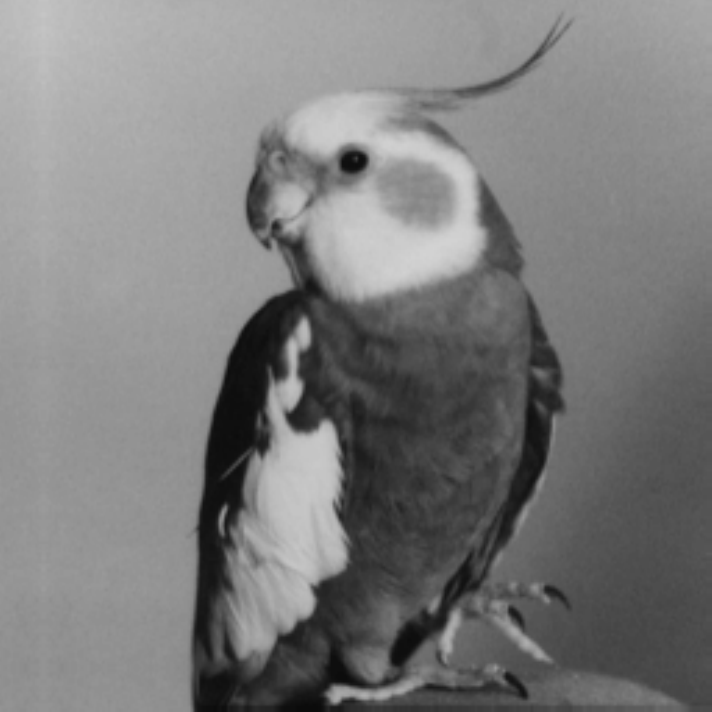}\label{histogram original image}} 
\subfigure[Original histogram]{\includegraphics[width=4.5cm, height=3.2cm]
{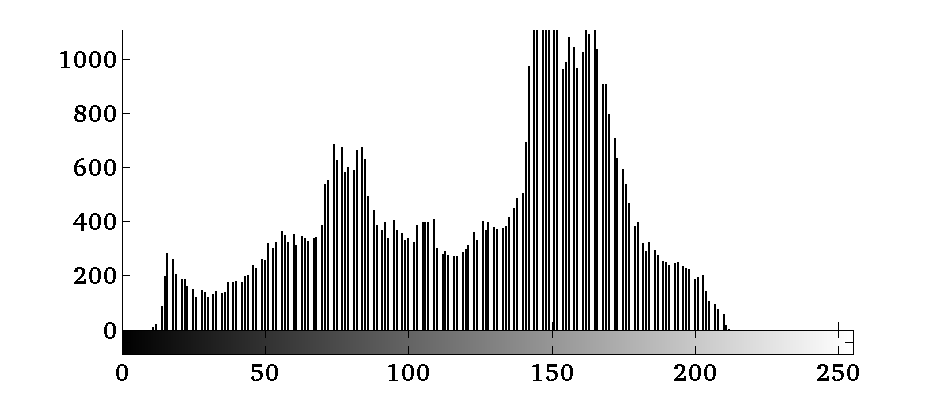}\label{Histograma imagenen de ejemplo}}
\subfigure[Ring addition]{\includegraphics[width=4.5cm, height=3.2cm]%
{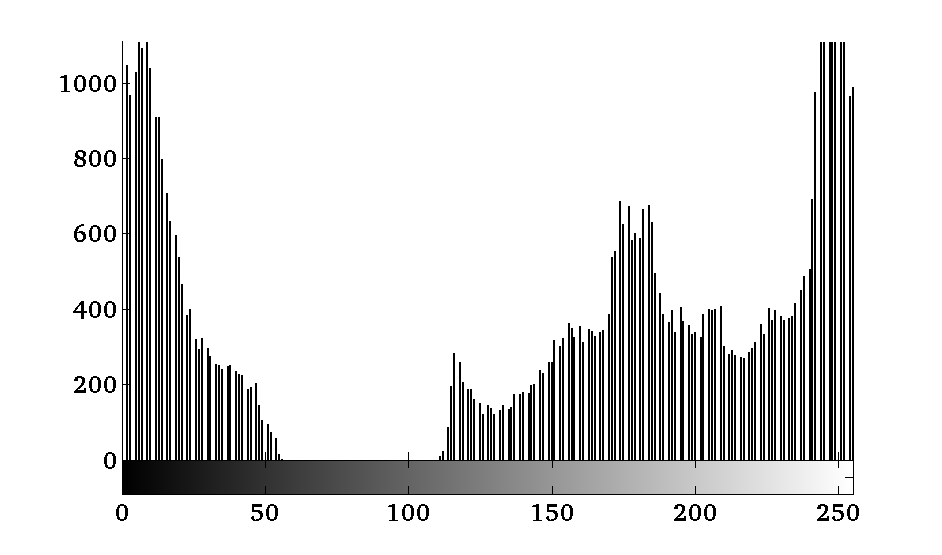}\label{frecuencia ring sum}}\\
\subfigure[Ring subtraction]{\includegraphics[width=4.5cm, height=3.2cm]
{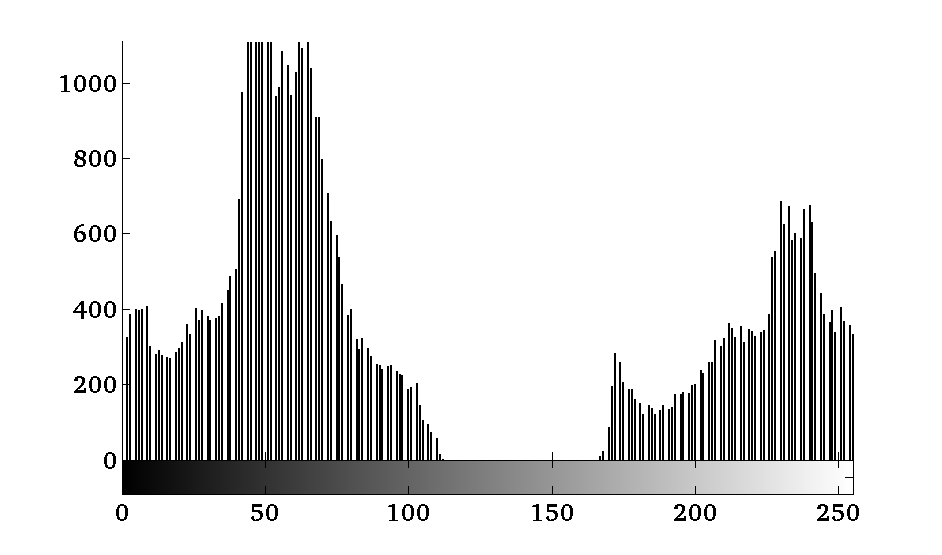}\label{frecuencia ring rest}}
\subfigure[Classical addition]{\includegraphics[width=4.5cm, height=3.2cm] 
{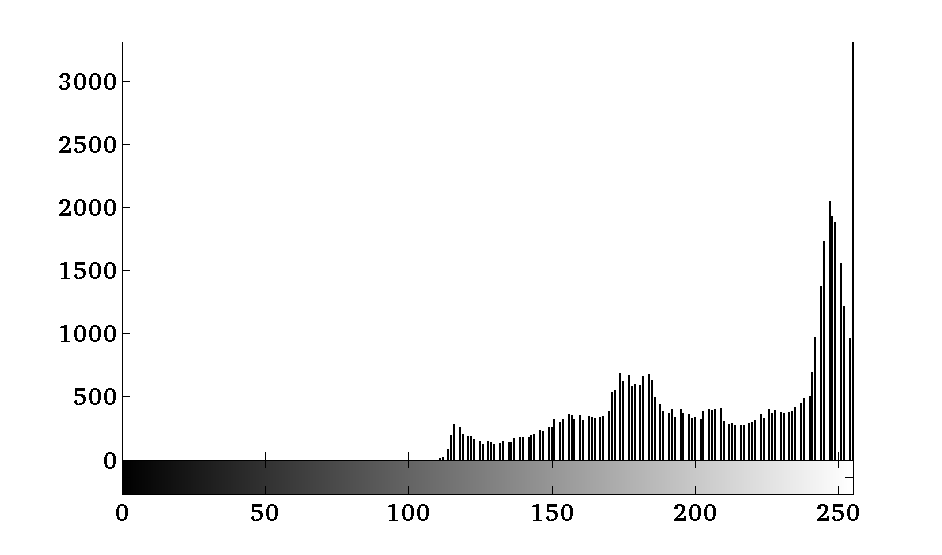}\label{frecuencia normal sum}}
\subfigure[Classical subtraction]{\includegraphics[width=4.5cm, height=3.2cm]
{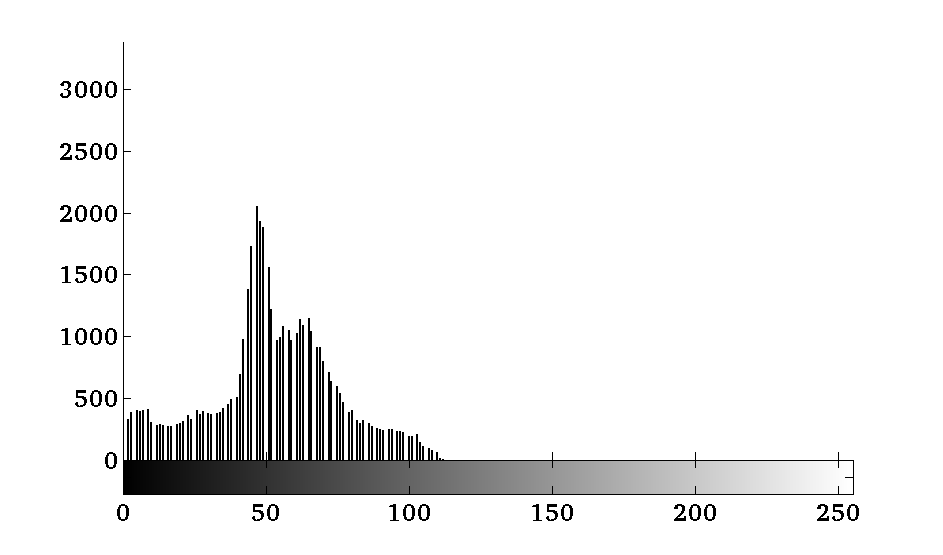}\label{frecuencia normal rest}}
\vspace{-.1in}
\caption{A sample image and his histogram of frecuency.}
\label{real histogram bird}
\vspace{-.25in}
\end{figure}

Note that the effect of this operations only changes the value of the pixels but do not changes the frecuency or the number of gray levels. In fact, graphically it is possible to see that the addition or subtraction in the ring only generates a shift in the frequency histogram, so, the images with the frecuency histogram in Figure \ref{frecuencia ring sum} and Figure \ref{frecuencia ring rest}, have the same properties and information that the original image in Figure \ref{histogram original image}. On the other hand, if we consider the classic addition and subtraction among images, it is easy to see in Figure \ref{frecuencia normal sum} and Figure \ref{frecuencia normal rest} that important information is lost due to the effect of truncation. In fact, this effect causes that high frequencies have been accumulated in $255$ in the addition case and in $0$ in the subtraction case (see Figures \ref{frecuencia normal sum} and \ref{frecuencia normal rest}).

Based on the above aspects, it is interesting to define the equivalence classes among images based on the concept of strong equivalence. This definition is necessary because the strong equivalent images has the same information, but from the point of view of intesity pixel, they do not have to be identically equal. In such sense, we will define the quotient group containing images in equivalent classes.

\begin{definition}[Equivalence Classes]%\label{equivalence class}
Given an image $\mathcal{A}$, we said that the image $\mathcal{B}$ is in the class of $\mathcal{A}$ if and only if $\mathcal{A}$ and $\mathcal{B}$ are strongly equivalents. We denote the equivalence class of $\mathcal{A}$ as $C_{\mathcal{A}}$.
\end{definition}

It is natural to consider the definition of the similarity index using the quotient space of the images by using strong equivalent images, considering the given equivalence classes above. For this reason, it is necessary to proof that the quotient space exists. More precisely, as strong equivalent images have been defined by the addition operation, the quotient group will be built on this operation.

\begin{theorem}
The set $N=\lbrace\mathcal{S}\in G_{k\times m} (\mathbb{Z}_{n})(+,\cdot),\ S\ \mbox{is a scalar image}\rbrace$ is a normal (or invariant) subgroup of $G_{k\times m}(\mathbb{Z}_{n})(+,\cdot)$.
\end{theorem}
\begin{proof}
First we will show that the set $N$ is a subgroup of $G_{k\times m} (\mathbb{Z}_{n})(+,\cdot)$.
\begin{enumerate}
\item{\textbf{Addition operation is well-defined}.}
{If $\mathcal{S},\ \mathcal{C}\in N \Rightarrow \forall\ i,j,i',j'(\ i,i'\leq k \wedge j,j'\leq m)$, $s_{i,j}=s_{i',j'} \wedge c_{i,j}=c_{i',j'}$ where $\lbrace s_{i,j},\ s_{i',j'}\rbrace$ are elements of $\mathcal{S}$ and $\lbrace c_{i,j},\ c_{i',j'}\rbrace$ are elements of $\mathcal{C}$. Then, $s_{i,j}+c_{i,j}=s_{i',j'}+c_{i',j'}$ and the image $\mathcal{S}+\mathcal{C}$ has all elements equals, so, $\mathcal{S}+\mathcal{C}\in N$}.
\item {\textbf{Neutral element}.}
{The neutral element ($\mathcal{O}$) in the group is the image with all pixels equals to zero, defined at the previous subsection \ref{Image via Ring Theory}.}
\item {\textbf{Inverse element}.}
{Let $\mathcal{S}\in N\Rightarrow \forall\ i,j,i',j',\  s_{i,j}=s_{i',j'}$, but $s_{i,j}\in \mathbb{Z}_{n}\Rightarrow s_{i,j}$ has a inverse element $s'_{i,j}$ in $\mathbb{Z}_{n}$. Considering the image $\mathcal{S}'$ with all elements equal to $s'_{i,j}$, we obtain that $\mathcal{S}+\mathcal{S}'=\mathcal{O}$, then $\mathcal{S}'$ is the inverse of $\mathcal{S}$.}
\end{enumerate}
We have proved that $N$ is a subgroup of $G_{k\times m}(\mathbb{Z}_{n})(+,\cdot)$, only remains to prove that $N$ is normal. For this purpose we use that $N$ is a normal subgroup of $G$ if $\forall g\in G,\ g+N=N+g$ (see \cite{Kostrikin,Neiderreiter,Teresita Noriega} for more details).

Let $\mathcal{A}\in\displaystyle\frac{G_{k\times m}(\mathbb{Z}_{n})(+,\cdot)}{N}$ and $\mathcal{S}\in N$, let us define $\mathcal{B}=\mathcal{A}+\mathcal{S}$ where $b_{i,j}=a_{i,j}+s_{i,j}\ \forall i,j$. But, we know that $\mathbb{Z}_{n}$ is abelian, therefore $G_{k\times m}(\mathbb{Z}_{n})(+,\cdot)$ is abelian too, so, $b_{i,j}=s_{i,j}+a_{i,j}\ \forall i,j, \Rightarrow \mathcal{B}=\mathcal{S}+\mathcal{A}$. Finally $\mathcal{A}+\mathcal{S}=\mathcal{S}+\mathcal{A}$.
\end{proof}

With the last result, it was proved that the quotient space is well-defined, and it will be denoted by $\displaystyle\frac{G_{k\times m} (\mathbb{Z}_{n})(+,\cdot)}{N}$. Remember that the elements of this quotient space are the equivalence classes. Now, it does not matter which element of the class is chosen to carry out the operation, since any of these images are a representative element of its respective equivalence class. 

\begin{definition}[Natural Entropy Distance]%\label{definition ned}
Let $C_{\mathcal{A}}$ and $C_{\mathcal{B}}$ be two elements in the quotient space $\displaystyle\frac{G_{k\times m} (\mathbb{Z}_{n})(+,\cdot)}{N}$, $\mathcal{A}_{1}\in C_{\mathcal{A}}$ and $\mathcal{B}_{1}\in C_{\mathcal{B}}$ are images. The Natural Entropy Distance (NED) between $\mathcal{A}_{1}$ and $\mathcal{B}_{1}$ is defined by
\begin{equation}
\label{equation ned}
\hat{\nu}(\mathcal{A}_{1},\mathcal{B}_{1})=E(\mathcal{A}_{1} + (-\mathcal{B}_{1})).
\end{equation}
\end{definition}

The ``Natural Entropy Distance'' have the following properties related with the axioms of distance:
\begin{enumerate}
\item{\textbf{Non-Negativity}: $\hat{\nu}(\mathcal{A}_1,\mathcal{B}_1)\geq 0$.\\
Since entropy function is always positive, then
$$
\hat{\nu}(\mathcal{A}_{1},\mathcal{B}_{1})=E(\mathcal{A}_{1} + (-\mathcal{B}_{1}))\geq 0.$$}
\item{\textbf{Identity of indiscernibles}: $\hat{\nu}(\mathcal{A}_1,\mathcal{B}_1)=0\Leftrightarrow C_{\mathcal{A}}=C_{\mathcal{B}}$.
\begin{eqnarray*}
\hat{\nu}(\mathcal{A}_1,\mathcal{B}_1)=0 &\Leftrightarrow & E(\mathcal{A}_{1} + (-\mathcal{B}_{1}))= 0\\ &\Leftrightarrow &\mathcal{A}_{1}+(-\mathcal{B}_{1})=\mathcal{S}, \ \mathcal{S}\in N\\
&\Leftrightarrow &\mathcal{B}_{1}\in C_{\mathcal{A}} \\
&\Leftrightarrow &C_{\mathcal{A}}=C_{\mathcal{B}}.
\end{eqnarray*}}
\item{\textbf{Reflexivity}: $\hat{\nu}(\mathcal{A}_1,\mathcal{B}_1)=\hat{\nu}(\mathcal{B}_1,\mathcal{A}_1)$.
\begin{eqnarray*}
\hat{\nu}(\mathcal{A}_1,\mathcal{B}_1) &=& E(\mathcal{A}_{1} + (-\mathcal{B}_{1}))\\
								&=& E(-(\mathcal{B}_{1} + (-\mathcal{A}_{1})))\\
								&=& E(\mathcal{B}_{1} + (-\mathcal{A}_{1}))\\
								&=& \hat{\nu}(\mathcal{B}_{1},\mathcal{A}_{1}).
\end{eqnarray*}}
\end{enumerate}

If it are considered the images of Figure \ref{different images comparison}, applying the natural entropy distance the result shows that:
\begin{align*}
\hat{\nu}(Figure\ \ref{one_a},\ Figure\ \ref{two_b})=0.2514.
\end{align*}
With this result, one can appreciate that have been differentiated images, including the spatial information too. The similarity index that was proposed in (\ref{equation ned}) is very simple and computationally efficient. 

\section{Natural Entropy Distance and $\mathit{MSHi}$}
\label{experiments}
Taking in consideration  the good properties that, in general, the NED definition has, one sees logical to take this new similarity index as the new stopping criterion of $\mathit{MSHi}$. Explicitly, the new stopping criterion is: 
\begin{equation}
\label{definition stop criterion new}
E(\mathcal{A}_{k}+(-\mathcal{A}_{k-1}))\leq \epsilon ,
\end{equation}
where $\epsilon$ and $k$ are respectively the threshold to stop the iterations and the number of iterations.

Since $\mathit{MSHi}$ is an iterative algorithm, we obtain a sequence that the  processed image becames more homogenous as the algorithm advances. It is intuitive that comparing a resultant image $\mathcal{A}_{k}$ at a given iteration with the previous image $\mathcal{A}_{k-1}$ of the corresponding sequence, give us a measure about how much these images look like each other. The Algorithm \ref{AlgoritmoNew} shows the general structure of the $\mathit{MSHi}$ with NED as stopping criterion. We denote this algorithm as $\mathit{MSHi}_{NED}$, and $\mathit{MSHi}_{WE}$ when the similarity index (\ref{old criterion}) is used as stopping criterion.
\begin{algorithm}[H]
\caption{$\mathit{MSHi}$ with Natural Entropy Distance ($\mathit{MSHi}_{NED}$)}\label{AlgoritmoNew}
%\SetLine
\KwData{Original Image ($\mathcal{A}$); Threshold to stop ($\epsilon$).}
 Initialize $\mathcal{B}_{1}=\mathcal{A}$ and $errabs=\infty$\;	
 \While{$errabs > \epsilon$}{
 \begin{itemize}
 \item{Filter the original image according to the mean shift algorithm (Section \ref{sub:Mean Shift});\\ store in $\mathcal{B}_{2}$ the filtered image;} 
 \item{Calculate using (\ref{definition stop criterion new}) the difference between $\mathcal{B}_{1}$ and $\mathcal{B}_{2}$, $errabs=\hat{\nu}(\mathcal{B}_{1},\mathcal{B}_{2})$.}
 \item{Update the images: $\mathcal{B}_{1} = \mathcal{B}_{2}$.}
 \end{itemize}\vspace{-.1in}}
\KwResult{$\mathcal{B}_{1}$ (is the segmented image).}
\end{algorithm}

The principal goal of this section is to evaluate the new stopping criterion in the $\mathit{MSHi}$ and to illustrate that, in general, with this new stopping criterion the algorithm has better stability in the segmentation process. For this aim, we used three different images for the experiments, which have been chosen according to the differences among their respective levels of high and low frequencies. The first image (``Bird") has low frequencies, the second (``Baboon") has high frequencies and in the image (``Montage") has mixture low and high frequencies.
\vspace{-.2in}
\begin{figure}[H]
\centering
\subfigure[Bird]{\includegraphics[width=3.5cm, height=3.5cm] {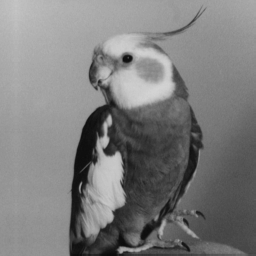}}
\subfigure[Baboon]{\includegraphics[width=3.5cm, height=3.5cm]
{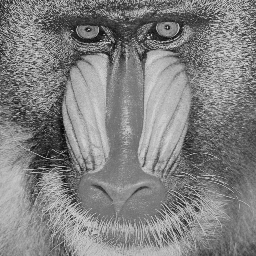}\label{baboon}}
\subfigure[Montage]{\includegraphics[width=3.5cm, height=3.5cm]
{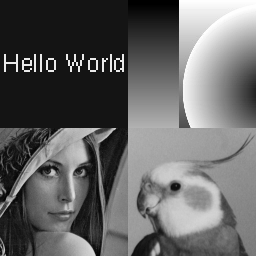}}
\vspace{-.15in}
\caption{Test images}
\vspace{-.25in}
\end{figure}

All segmentation experiments were carried out by using a uniform kernel. In order to be effective the comparison between the old stopping criterion and the new stopping criterion, we use the same value of $h_{r}$ and $h_{s}$ in $\mathit{MSHi}$ ($h_{r}=12,\ h_{s}=15$). The value of $h_{s}$ is related to the spatial resolution of the analysis, while the value $h_{r}$ defines the range resolution. In the case of the new stopping criterion, we use the stopping threshold $\epsilon= 0.9$ and when the old stopping criterion was used $\epsilon= 0.01$.

Figure \ref{test images} shows the segmentation of the three images. Observe that, in all cases, the $\mathit{MSHi}$ had better result when the new stopping criterion was used.

\begin{figure}[H]
\vspace{-.2in}
\centering
\begin{tabular}{ccc}\hline
Bird & Baboom & Montage \\\hline
\subfigure[$\mathit{MSHi}_{NED}$]{\includegraphics[width=4.09cm, height=4.09cm] {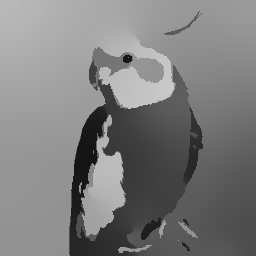}\label{bird_new}} &
\subfigure[$\mathit{MSHi}_{NED}$]{\includegraphics[width=4.1cm, height=4.1cm]
{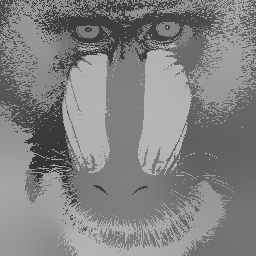}\label{baboon_new}} &
\subfigure[$\mathit{MSHi}_{NED}$]{\includegraphics[width=4.1cm, height=4.1cm] {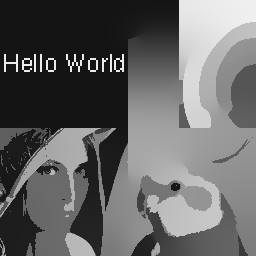}\label{montage_new}}\\
\subfigure[$\mathit{MSHi}_{WE}$]{\includegraphics[width=4.09cm, height=4.09cm]
{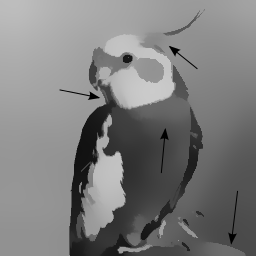}\label{bird_old} } &
\subfigure[$\mathit{MSHi}_{WE}$]{\includegraphics[width=4.1cm, height=4.1cm]
{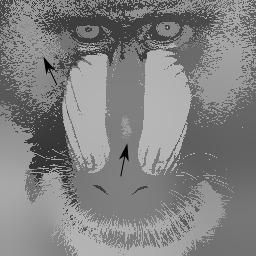}\label{baboon_old}} &
\subfigure[$\mathit{MSHi}_{WE}$]{\includegraphics[width=4.1cm, height=4.1cm]
{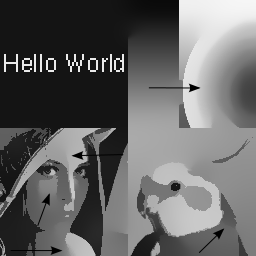}\label{montage_old}}\\\hline
\end{tabular}
\caption{Segmentation of the experimental images. In the first row are shown the segmentation using the new stopping criterion and in the secound row are the segmented images using the old stopping criterion.}
\label{test images}
\vspace{-.5in}
\end{figure}

When one compares Figures \ref{bird_new} and \ref{bird_old}, in the part corresponding to the face or breast of the bird, more homogeneous area, with the new stopping criterion was obtained (see arrows in Figure \ref{bird_old}). Observe that, with the old stopping criterion the segmentation gives regions where different gray levels are originated. However, these regions really should have only one gray level. For example, Figure \ref{baboon_new} and \ref{baboon_old} show that the segmentation is more homogeneous when the new stopping criterion was used (see the arrows). In the case of the ``Montage" image one can see that, in Figure \ref{montage_old} exists many regions that contains different gray levels when these regions really should have one gray level (see for example the face of Lenna, the circles and the breast of the bird). These good results are obtained because the defined new stopping criterion  through the natural distance among images in expression (\ref{definition stop criterion new}) offers greater stability to the $\mathit{MSHi}$.

Figure \ref{profile test} shows the profile of the obtained segmented images by using the two stopping criterion. The plates that appear in Figure \ref{profile new} and \ref{profile old} are indicative of equal intensity levels. In both graphics the abrupt falls of an intensity to other represent the different regions in the segmented image. Note that, in Figure \ref{profile new} exists, in the same region of the segmentation, least variation of the pixel intensities with regard to Figure \ref{profile old}. This illustrates that, in this case the segmentation was better when the new stopping criterion was used. 
\begin{figure}[h]
\vspace{-.2in}
\centering
\subfigure[New Criterion]{\includegraphics[scale=0.5]
{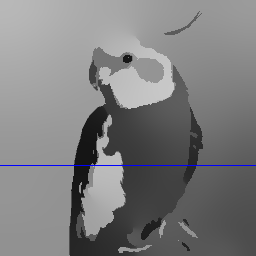}}
\subfigure[Profile]{\includegraphics[scale=0.5]{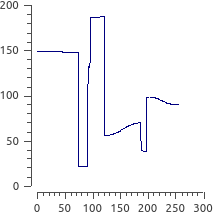} 
\label{profile new}}
\subfigure[Old Criterion]{\includegraphics[scale=0.5]
{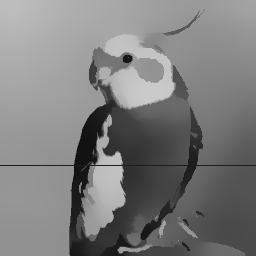}}
\subfigure[Profile]{\includegraphics[scale=0.5]{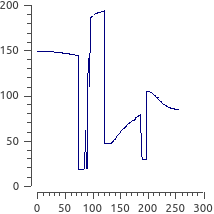}  
\label{profile old}}
\vspace{-.1in}
\caption{An intensity profile through of a segmented image. Profile is indicated by a line. (a) and (c) are the segmented images and (b) and (d) are the profile of (a) and (c) respectively.}
\label{profile test}
\vspace{-.3in}
\end{figure}

Figure \ref{iterations} shows the performance of the two stopping criterion in the experimental images. In the ``$x$" axis appears the iterations of $\mathit{MSHi}$ and in the ``$y$" axis is shown the obtained values by the stopping criterion in each iteration of the algorithm.

\begin{figure}[h]
\vspace{-.2in}
\centering
\begin{tabular}{ccc}\hline
Bird & Baboom & Montage\\\hline
\subfigure[$\mathit{MSHi}_{NED}$]{\includegraphics[width=3.5cm,height=3.5cm] {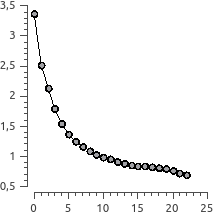}\label{iterations_new_bird} } &
\subfigure[$\mathit{MSHi}_{NED}$]{\includegraphics[width=3.5cm, height=3.5cm]
{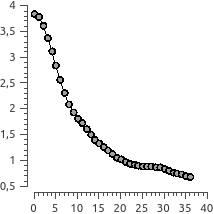}\label{iterations_new_baboon}} &
\subfigure[$\mathit{MSHi}_{NED}$]{\includegraphics[width=3.5cm, height=3.5cm]
{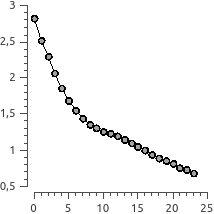}\label{iterations_new_montage}}\\
\subfigure[$\mathit{MSHi}_{WE}$]{\includegraphics[width=3.5cm, height=3.5cm]
{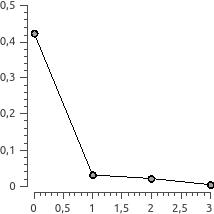}\label{iterations_old_bird}} &
\subfigure[$\mathit{MSHi}_{WE}$]{\includegraphics[width=3.5cm, height=3.5cm]
{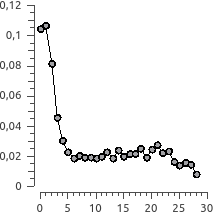}\label{iterations_old_baboon}} &
\subfigure[$\mathit{MSHi}_{WE}$]{\includegraphics[width=3.5cm, height=3.5cm] {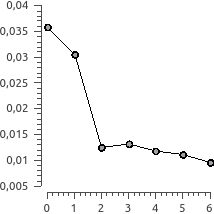}\label{iterations_old_montage}}\\\hline
\end{tabular}
\vspace{-.1in}
\caption{Behavior of the stability of each algorithm according to the test images.}
\label{iterations}
\vspace{-.3in}
\end{figure}

The graphics of iterations of the new stopping criterion (Figure \ref{iterations_new_bird}, \ref{iterations_new_baboon}, \ref{iterations_new_montage}) show a smoother behavior; that is, the new stopping criterion has a stable performance through the $\mathit{MSHi}$. The new stopping criterion not only has good theoretical properties, but also, in the practice, has very good behavior.

On the other hand, if we analyze the old stopping criterion in the experimental images (Figure \ref{iterations_old_bird}, \ref{iterations_old_baboon}, \ref{iterations_old_montage}), one can see that the performance in the $\mathit{MSHi}$ is unstable. In general, we have this type of situation when the stopping criterion defined in (\ref{old criterion}) is used. This can originate bad segmented images.

\section{Conclusions}
\label{conclusion}
In this work, a new stopping criterion, for the iterative $\mathit{MSHi}$, based on the Ring Theory was proposed. The new stopping criterion establishes a new measure for the comparison of two images based on the use of the entropy concept and the spatial information. The quotient space was defined using the equivalent classes of images, to be able of selecting any element of the class. Through the obtained theoretical and practical results, it was possible to prove that the new stopping criterion had very good performance in the algorithm $\mathit{MSHi}$, and was more stable that the old criterion.

%\section*{References}

%References are to be listed alphabetically with Arabic numerals.  They
%can be typed in superscripts after punctuation marks, e.g.~``$\ldots$
%in the statement.\cite{joliat}'' or used directly, e.g.~``see
%Ref.~\refcite{clark} for examples.'' Please list using the style shown
%in the following examples.  For journal names, use the standard
%abbreviations.  Typeset references in 9~pt roman.

%\vspace*{-0.01in}
%%\vspace*{-0.3in}
%\noindent
%\rule{12.6cm}{.1mm}
%
%\section*{Biographical Sketch and Photo}
%
%Upon acceptance of an article, a brief biographical sketch and  
%photograph of each author are to be supplied to the Publisher.
%
%\biophoto{wang}{{\bf Chuan-Cheng Wang} received the 
%B.S.~degree\break 
%in electrical engineering from National Sun Yat-Sen University,
%Kaoh-\break siung, Taiwan in 1992 and the M.S. degree in
%computer science from National Chiao Tung University, Hsinchu,\break 
%Taiwan in 2001.}
%
%\vglue-1.75truein
%\hspace*{2.45truein}
%\biophoto{gao}{{\bf Yongsheng Gao} received the B.Sc. and\break
%M.Sc. degrees in electronic engineering from\break Zhejiang
%University,\break China, in 1985 and 1988 respectively, and the
%Ph.D. in computer engineering from Nanyang Technological University,
%Singapore. Currently, he is an assistant professor with Nanyang
%Technological University, Singapore. }

\end{document}